\documentclass[10pt,twocolumn,letterpaper]{article}

\usepackage{cvpr}              

\definecolor{cvprblue}{rgb}{0.21,0.49,0.74}
\usepackage[pagebackref,breaklinks,colorlinks,allcolors=cvprblue]{hyperref}

\usepackage[accsupp]{axessibility}  

\usepackage{amsmath}
\usepackage{amssymb}
\DeclareMathOperator*{\argmin}{argmin}

\usepackage[ruled, vlined, linesnumbered]{algorithm2e} 
\SetKwComment{PythonStyleComment}{\# }{}

\SetCommentSty{mycommfont}

\SetKwFor{SimpleRepeat}{repeat}{do}{endrepeat}
\SetKwInput{KwHyperParam}{Hyper params}

\usepackage{graphicx}
\graphicspath{{./img/}}

\usepackage[frozencache=true,cachedir=minted-cache]{minted}

\usepackage{subcaption}

\usepackage{wrapfig}

\usepackage{amsthm} 
\theoremstyle{plain}
\newtheorem{theorem}{Theorem}[section]
\newtheorem{proposition}[theorem]{Proposition}

\theoremstyle{definition}

\theoremstyle{remark}

\def\paperID{16509} 
\def\confName{CVPR}
\def\confYear{2025}

\title{LotusFilter: Fast Diverse Nearest Neighbor Search via a Learned Cutoff Table}

\author{Yusuke Matsui\\
The University of Tokyo\\
{\tt\small matsui@hal.t.u-tokyo.ac.jp}
}

\begin{document}

\maketitle

\begin{abstract}
Approximate nearest neighbor search (ANNS) is an essential building block for applications like RAG but can sometimes yield results that are overly similar to each other. In certain scenarios, search results should be similar to the query and yet diverse. We propose LotusFilter, a post-processing module to diversify ANNS results. We precompute a cutoff table summarizing vectors that are close to each other. During the filtering, LotusFilter greedily looks up the table to delete redundant vectors from the candidates. We demonstrated that the LotusFilter operates fast (0.02 [ms/query]) in settings resembling real-world RAG applications, utilizing features such as OpenAI embeddings. Our code is publicly available at \url{https://github.com/matsui528/lotf}.
\end{abstract}

\section{Introduction}

An approximate nearest neighbor search (ANNS) algorithm, which finds the closest vector to a query from database vectors~\cite{book_bruch2024,tutorial_matsui2020,tutorial_matsui2023}, is a crucial building block for various applications, including image retrieval and information recommendation. Recently, ANNS has become an essential component of Retrieval Augmented Generation (RAG) approaches, which integrate external information into Large Language Models~\cite{tutorial_asai2023}.

\begin{figure}[t]
  \centering
  \begin{minipage}[b]{0.48\hsize}
    \centering
    \includegraphics[width=1.0\linewidth]{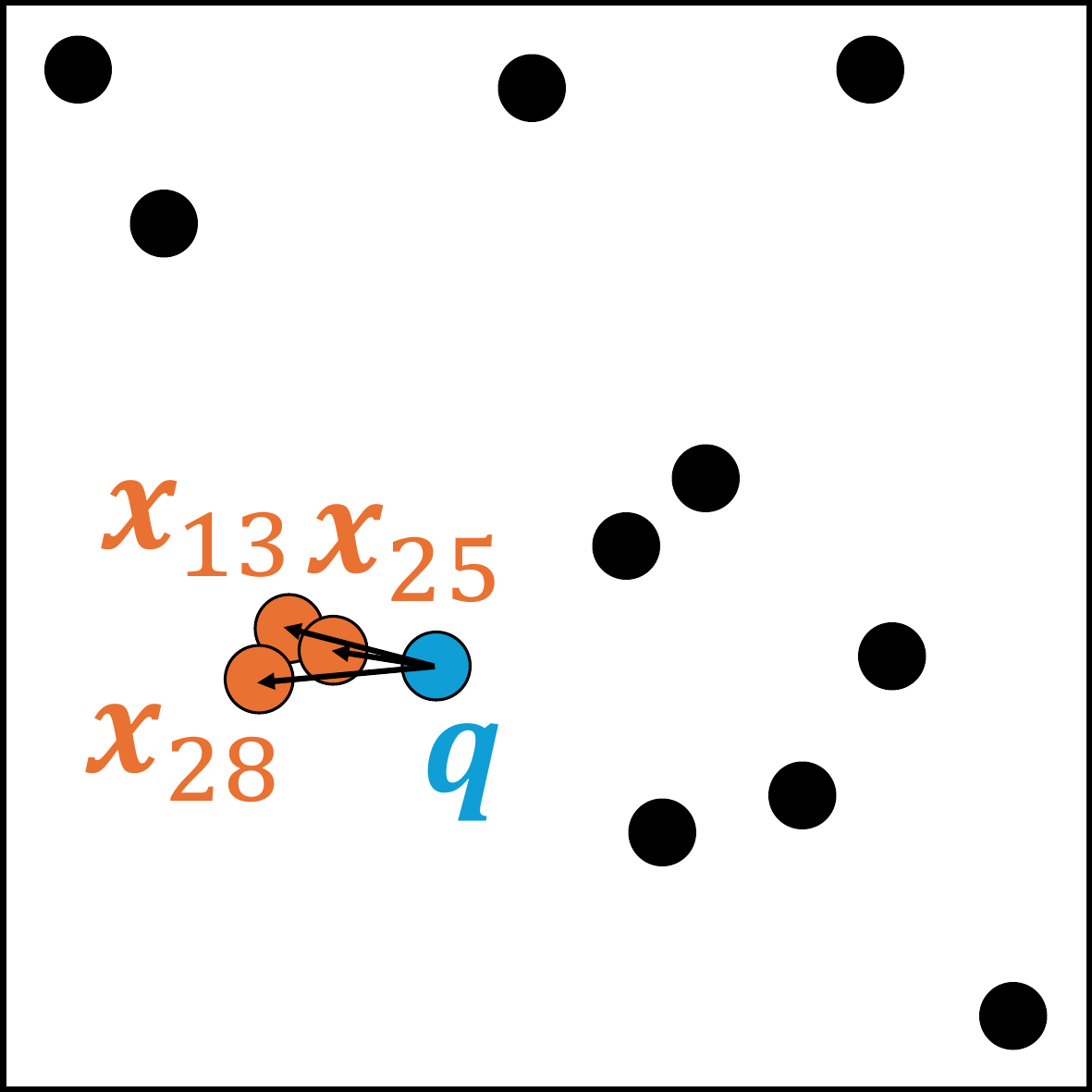}
    \subcaption{Usual ANNS}\label{fig:teaser1}
  \end{minipage}
  \hfill
  \begin{minipage}[b]{0.48\hsize}
    \centering
    \includegraphics[width=1.0\linewidth]{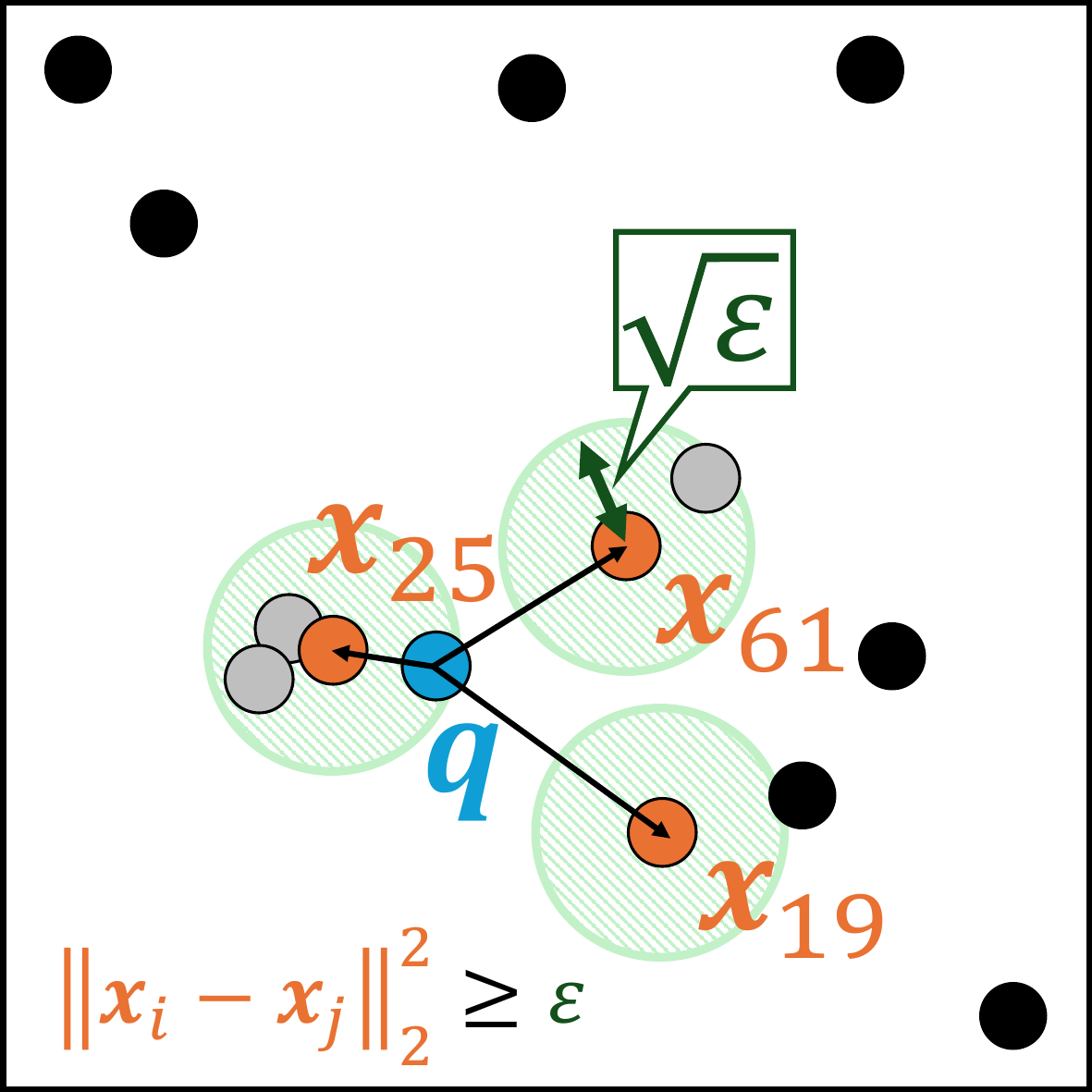}
    \subcaption{DNNS with the LotusFilter}\label{fig:teaser2}
  \end{minipage}
  \caption{(a) Usual ANNS. The search results are close to the query $\mathbf{q}$ but similar to each other. (b) DNNS with the proposed LotusFilter. The obtained vectors are at least $\sqrt{\varepsilon}$ apart from each other. The results are diverse despite being close to the query.}\label{fig:teaser}
\end{figure}

The essential problem with ANNS is the lack of diversity. For example, consider the case of image retrieval using ANNS. Suppose the query is an image of a cat, and the database contains numerous images of the same cat. In that case, the search results might end up being almost uniform, closely resembling the query. However, users might prefer more diverse results that differ from one another.

Diverse nearest neighbor search (DNNS)~\cite{sigmod_drosou2010, ftir_santos2015,kis_zheng2017} is a classical approach to achieving diverse search results but often suffers from slow performance. Existing DNNS methods first obtain $S$ candidates (search step) and then select $K(<S)$ results to ensure diversity (filter step). This approach is slow for three reasons. First, integrating modern ANN methods is often challenging. Second, selecting $K$ items from $S$ candidates is a subset selection problem, which is NP-hard. Lastly, existing methods require access to the original vectors during filtering, which often involves slow disk access if the vectors are not stored in memory.

We propose a fast search result diversification approach called \textit{LotusFilter}, which involves precomputing a cutoff table and using it to filter search results. Diverse outputs are ensured by removing vectors too close to each other. The data structure and algorithm are both simple and highly efficient (\cref{fig:teaser}), with the following contributions:

\begin{itemize}
\item As LotusFilter is designed to operate as a pure post-processing module, one can employ the latest ANNS method as a black-box backbone. This design provides a significant advantage over existing DNNS methods.
\item We introduce a strategy to train the hyperparameter, eliminating the need for complex parameter tuning.  
\item LotusFilter demonstrates exceptional efficiency for large-scale datasets, processing queries in only 0.02 [ms/query] for $9\times 10^5$ $1536$-dimensional vectors.  
\end{itemize}

\section{Related work}

\subsection{Approximate nearest neighbor search}

Approximate nearest neighbor search (ANNS) has been extensively studied across various fields~\cite{tutorial_matsui2020,tutorial_matsui2023}. Since around 2010, inverted indices~\cite{tpami_jegou2011,icassp_matsui2022,cvpr_douze2018,eccv_baranchuk2018,tpami_andre2021} and graph-based indices~\cite{tpami_malkov2020,mm_ono2023,sisap_oguri2023,vldb_fu2019,vldb_wang2021,neurips_subramanya2019} have become the standard, achieving search times under a millisecond for datasets of approximately $10^6$ items. These modern ANNS methods are significantly faster than earlier approaches, improving search efficiency by orders of magnitude.

\subsection{Diverse nearest neighbor search}

The field of recommendation systems has explored diverse nearest neighbor search (DNNS), especially during the 2000s~\cite{sigmod_drosou2010,ftir_santos2015,kis_zheng2017,sigir_carbonell1998}. Several approaches propose dedicated data structures as solutions~\cite{vldb_drosou2012,icmr_rao2016}, indicating that modern ANNS methods have not been fully incorporated into DNNS. Hirata et al. stand out as the only ones to use modern ANNS for diverse inner product search~\cite{recsys_hirata2022}.

Most existing DNNS methods load $S$ initial search results (the original $D$-dimensional vectors) and calculate all possible combinations even if approximate. This approach incurs a diversification cost of at least $\mathcal{O}(DS^2)$. In contrast, our LotusFilter avoids loading the original vectors or performing pairwise computations, instead scanning $S$ items directly. This design reduces the complexity to $\mathcal{O}(S)$, making it significantly faster than traditional approaches.

\subsection{Learned data structure}

Learned data structures~\cite{sigmod_kraska2018,inbook_ferragina2020} focus on enhancing classical data structures by integrating machine learning techniques. This approach has been successfully applied to well-known data structures such as B-trees~\cite{iclr_chen2023,icml_ferragina2020,vldb_ferragina2020,vldb_wu2021}, KD-trees~\cite{vldb_ding2020,sigmod_nathan2020,neurips_ws_hidaka2024}, and Bloom Filters~\cite{iclr_vaidya2021,neurips_sato2023,neurips_mitzenmacher2018,vldb_liu2020}. Our proposed method aligns with this trend by constructing a data structure that incorporates data distribution through learned hyperparameters for thresholding, similar to \cite{iclr_chen2023}.

\section{Preliminaries}

Let us describe our problem setting. Considering that we have $N$ $D$-dimensional database vectors $\{\mathbf{x}_n\}_{n=1}^N$, where $\mathbf{x}_n \in \mathbb{R}^D$. Given a query vector $\mathbf{q} \in \mathbb{R}^D$, our task is to retrieve $K$ vectors that are similar to $\mathbf{q}$ yet diverse, i.e., dissimilar to each other. We represent the obtained results as a set of identifiers, $\mathcal{K} \subseteq \{1, \dots, N \}$, where $|\mathcal{K}|=K$.

The search consists of two steps. First, we run ANNS and obtain $S (\ge K)$ vectors close to $\mathbf{q}$. These initial search results are denoted as $\mathcal{S}  \subseteq \{1, \dots, N \}$, where $|\mathcal{S}|=S$. The second step is diversifying the search results by selecting a subset $\mathcal{K} (\subseteq \mathcal{S})$ from the candidate set $\mathcal{S}$. This procedure is formulated as a subset selection problem. The objective here is to minimize the evaluation function $f \colon 2^\mathcal{S} \to \mathbb{R}$.
\begin{equation}
\argmin_{\mathcal{K} \subseteq \mathcal{S}, ~ |\mathcal{K}| = K} f(\mathcal{K}).
\label{eq:dnn}
\end{equation}
Here, $f$ evaluates how good $\mathcal{K}$ is, regarding both ``proximity to the query'' and ``diversity'', formulated as follows.
\begin{equation}
    f(\mathcal{K}) =
    \frac{1 - \lambda}{K} \sum_{k \in \mathcal{K}} \Vert \mathbf{q} - \mathbf{x}_k \Vert_2^2 -
    \lambda \min_{i, j \in \mathcal{K}, ~ i \ne j} \Vert \mathbf{x}_i - \mathbf{x}_j \Vert_2^2.
    \label{eq:f}
\end{equation}
The first term is the objective function of the nearest neighbor search itself, which indicates how close $\mathbf{q}$ is to the selected vectors. The second term is a measure of the diversity. Following \cite{aaai_amagata2023,recsys_hirata2022}, we define it as the closest distance among the selected vectors. Here $\lambda \in [0, 1]$ is a parameter that adjusts the two terms. If $\lambda=0$, the problem is a nearest neighbor search. If $\lambda=1$, the equation becomes the MAX-MIN diversification problem~\cite{or_ravi1994} that evaluates the diversity of the set without considering a query. This formulation is similar to the one used in \cite{recsys_hirata2022,icmr_rao2016,sigir_carbonell1998} and others.

Let us show the computational complexity of \cref{eq:dnn} is $\mathcal{O}(T + \binom{S}{K}DK^2)$, indicating that it is slow.
First, since it's not easy to represent the cost of ANNS, we denote ANNS's cost as $\mathcal{O}(T)$, where $T$ is a conceptual variable governing the behavior of ANNS. The first term in \cref{eq:f} takes $\mathcal{O}(DK)$, and the second term takes  $\mathcal{O}(DK^2)$ for a naive pairwise comparison. When calculating \cref{eq:dnn} naively, it requires $\binom{S}{K}$ computations for subset enumeration. Therefore, the total cost is $\mathcal{O}(T + \binom{S}{K}DK^2)$.

There are three main reasons why this operation is slow. First, it depends on $D$, making it slow for high-dimensional vectors since it requires maintaining and scanning original vectors. Second, the second term calculates all pairs of elements in $\mathcal{K}$ (costing $\mathcal{O}(K^2)$), which becomes slow for large $K$. Lastly, subset enumeration, $\binom{S}{K}$, is unacceptably slow.
In the next section, we propose an approximate and efficient solution with a complexity of $\mathcal{O}(T + S + KL)$, where $L$ is typically less than 100 for $N=9 \times 10^5$.

\section{LotusFilter Algorithm}

In this section, we introduce the algorithm of the proposed LotusFilter. The basic idea is to pre-tabulate the neighboring points for each $\mathbf{x}_n$ and then greedily prune candidates by looking up this table during the filtering step. 

Although LotusFilter is extremely simple, it is unclear whether the filtering works efficiently. Therefore, we introduce a data structure called OrderedSet to achieve fast filtering with a theoretical guarantee.

\subsection{Preprocessing}

\cref{alg:index} illustrates a preprocessing step. The inputs consist of database vectors $\{\mathbf{x}_n \}_{n=1}^N$ and the threshold for the squared distance, $\varepsilon \in \mathbb{R}$. In \texttt{L1}, we first construct $\mathcal{I}$, the index for ANNS. Any ANNS methods, such as HNSW~\cite{tpami_malkov2020} for faiss~\cite{arXiv_douze2024}, can be used here.

Next, we construct a cutoff table in \texttt{L2-3}. For each $\mathbf{x}_n$, we collect the set of IDs whose squared distance from $\mathbf{x}_n$ is less than $\varepsilon$. The collected IDs are stored as $\mathcal{L}_n$. We refer to these $\{ \mathcal{L}_n \}_{n=1}^N$ as a cutoff table (an array of integer arrays).

We perform a range search for each $\mathbf{x}_n$ to create the cutoff table. Assuming that the cost of the range search is also $\mathcal{O}(T)$, the total cost becomes $\mathcal{O}(NT)$. As demonstrated later in \cref{tbl:opt}, the runtime for $N = 9 \times 10^5$ is approximately one minute at most.

\begin{algorithm}[tb]
\DontPrintSemicolon
\KwIn{$\{\mathbf{x}_n\}_{n=1}^N \subseteq \mathbb{R}^D, ~ \varepsilon \in \mathbb{R}$}
$\mathcal{I} \gets \textsc{BuildIndex}(\{\mathbf{x}_n\}_{n=1}^N)$ \PythonStyleComment*[r]{ANNS} 
\For{$n \in \{1, \dots, N\}$}{
\hspace{-0.1cm}$\mathcal{L}_n \gets \{ i \in \{1, \dots, N \} \mid \Vert \mathbf{x}_n - \mathbf{x}_i \Vert_2^2 < \varepsilon, n \ne i\}$
}
\Return{$\mathcal{I}, \{\mathcal{L}_n\}_{n=1}^N$}
\caption{\textsc{Build}}
\label{alg:index}
\end{algorithm}

\begin{algorithm}[tb]
\DontPrintSemicolon
\KwIn{$\mathbf{q} \in \mathbb{R}^D, ~ S, ~ K (\le S), ~ \mathcal{I}, ~ \{\mathcal{L}_n\}_{n=1}^N$}
$\mathcal{S} \gets \mathcal{I}.\textsc{Search}(\mathbf{q}, S)$ \PythonStyleComment*[r]{$\mathcal{S} \subseteq \{1, \dots, N \}$}
$\mathcal{K} \gets \varnothing$ \; 
\While(\PythonStyleComment*[f]{At most $K$ times}){$|\mathcal{K}| < K$}{
$k \gets \textsc{Pop}(\mathcal{S})$ \PythonStyleComment*[r]{$\mathcal{O}(L)$} 
$\mathcal{K} \gets \mathcal{K} \cup \{k\}$ \PythonStyleComment*[r]{$\mathcal{O}(1)$} 
$\mathcal{S} \gets \mathcal{S} \setminus \mathcal{L}_k$ \PythonStyleComment*[r]{$\mathcal{O}(L)$}
}
\Return{$\mathcal{K}$} \PythonStyleComment*[r]{$\mathcal{K} \subseteq \mathcal{S}$ where $|\mathcal{K}| = K$}
\caption{\textsc{Search and Filter}}
\label{alg:divsearch}
\end{algorithm}

\subsection{Search and Filtering}

\begin{figure*}[tb]
  \begin{minipage}[b]{0.32\hsize}
    \centering
    \includegraphics[width=1.0\linewidth]{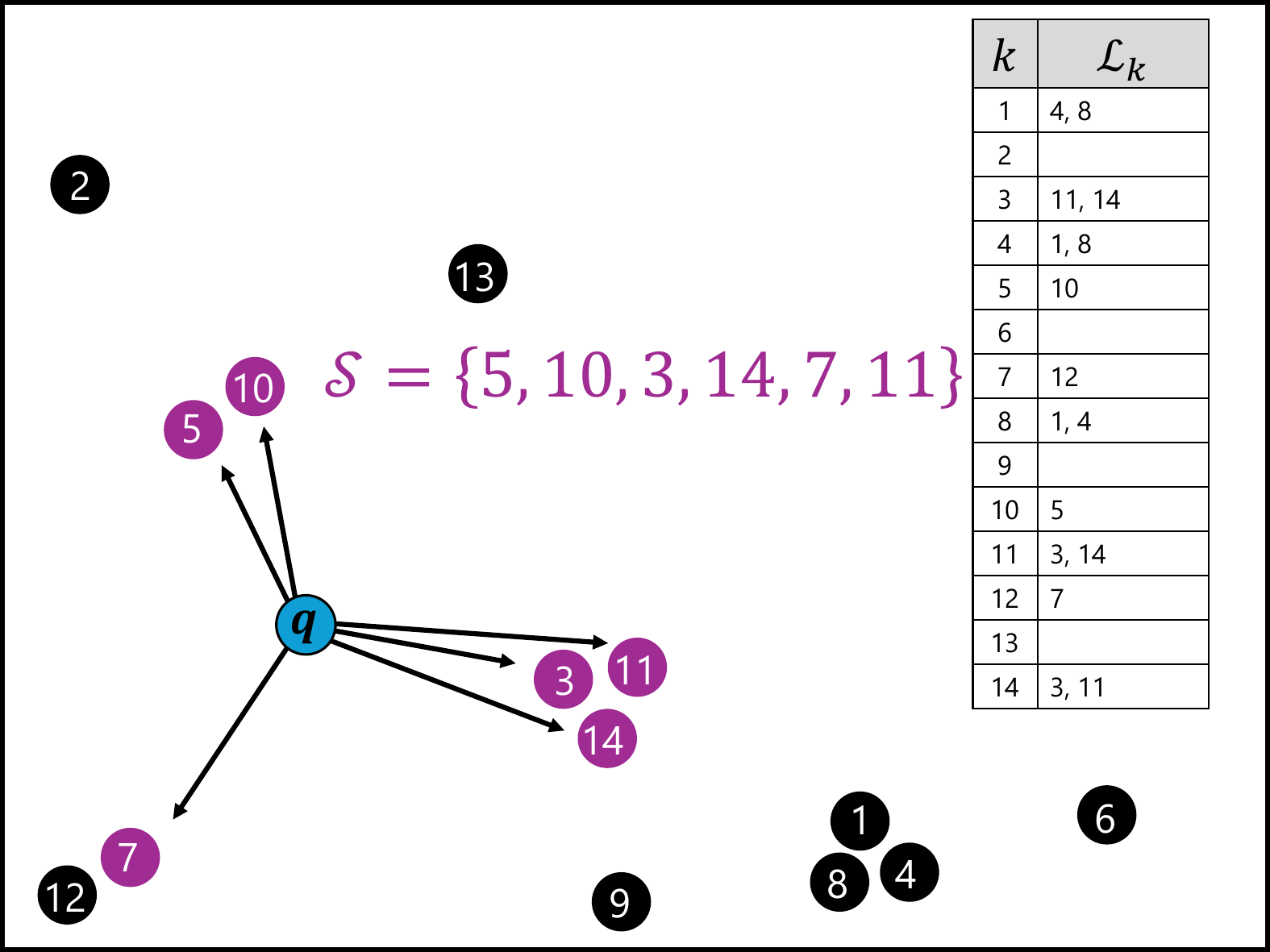}
    \subcaption{Initial search result}\label{fig:flow1}
  \end{minipage}
  \hfill
  \begin{minipage}[b]{0.32\hsize}
    \centering
    \includegraphics[width=1.0\linewidth]{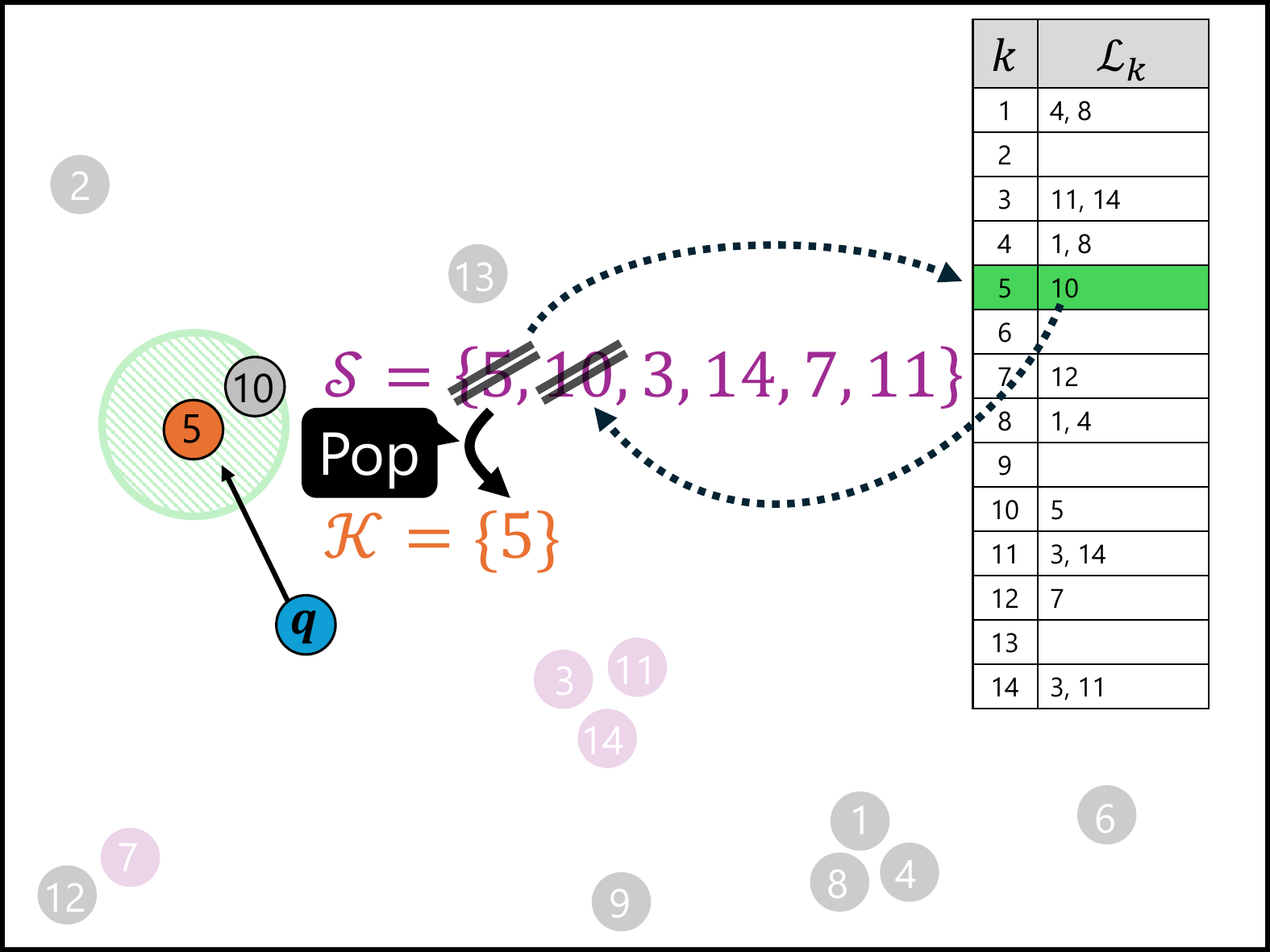}
    \subcaption{Accept the 1\textsuperscript{st} candidate. Cutoff.}\label{fig:flow2}
  \end{minipage}
  \hfill
  \begin{minipage}[b]{0.32\hsize}
    \centering
    \includegraphics[width=1.0\linewidth]{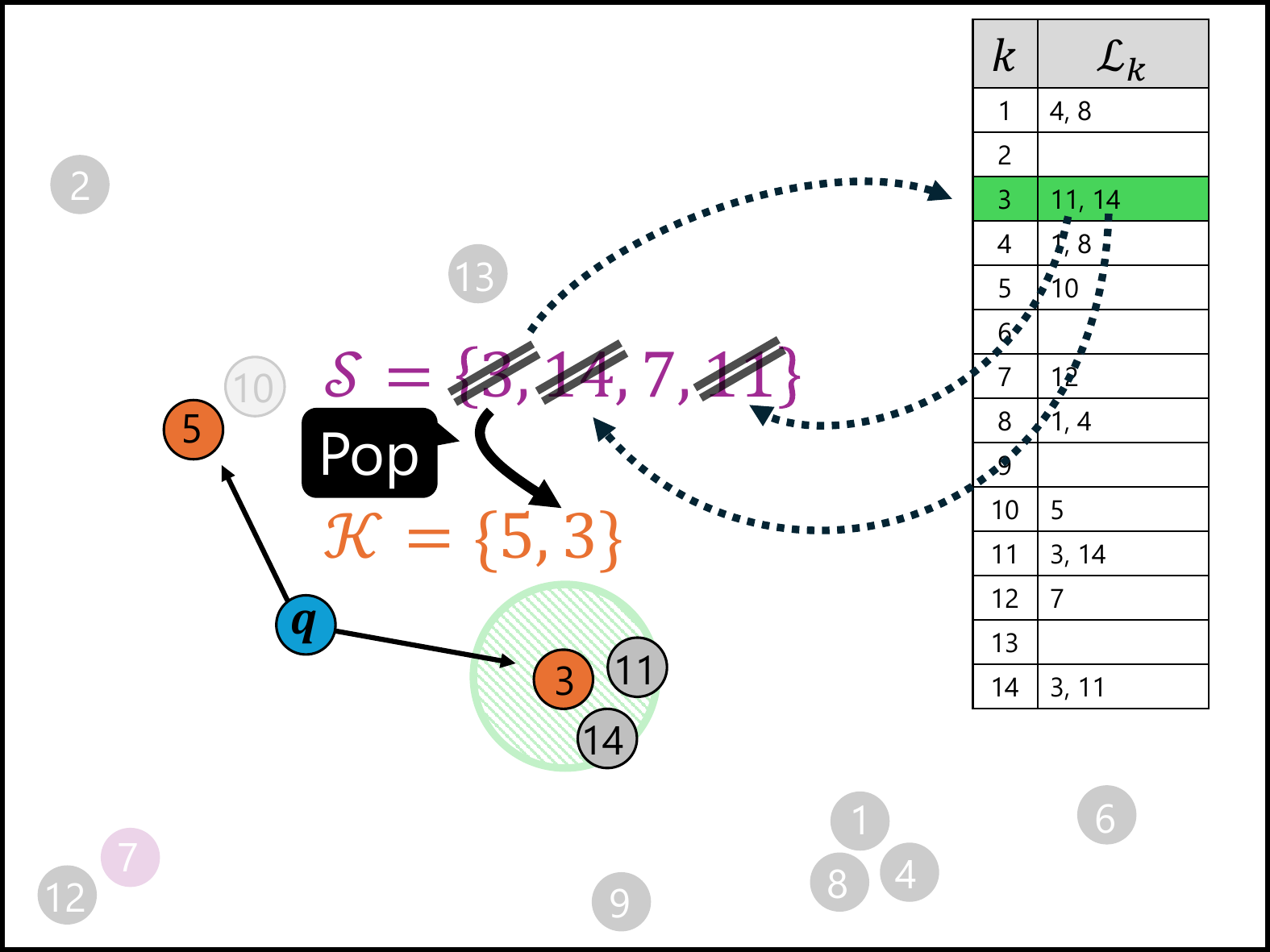}
    \subcaption{Accept the 2\textsuperscript{nd} candidate. Cutoff.}\label{fig:flow3}
  \end{minipage}
  \caption{Overview of the proposed LotusFilter ($D=2,~N=14,~S=6,~K=2$)}\label{fig:flow}
\end{figure*}

The search and filtering process is our core contribution and described in \cref{alg:divsearch} and \cref{fig:flow}. The inputs are a query $\mathbf{q} \in \mathbb{R}^D$, the number of initial search results $S (\le N)$, the number of final results $K (\le S)$, the ANNS index $\mathcal{I}$, and the cutoff table $\{\mathcal{L}_n\}_{n=1}^N$.

As the search step, we first run ANNS in \texttt{L1} (\cref{fig:flow1}) to obtain the candidate set $\mathcal{S} \subseteq \{1, \dots, N \}$. In \texttt{L2}, we prepare an empty integer set $\mathcal{K}$ to store the final results. 

The filtering step is described in \texttt{L3-6} where IDs are added to the set $\mathcal{K}$ until its size reaches $K$. In \texttt{L4}, we pop the ID $k$ from $\mathcal{S}$, where $\mathbf{x}_k$ is closest to the query, and add it to $\mathcal{K}$ in \texttt{L5}. Here, \texttt{L6} is crucial: for the current focus $k$, the IDs of vectors close to $\mathbf{x}_k$ are stored in $\mathcal{L}_k$. Thus, by removing $\mathcal{L}_k$ from $\mathcal{S}$, we can eliminate vectors similar to $\mathbf{x}_k$ (\cref{fig:flow2}). Repeating this step (\cref{fig:flow3}) ensures that elements in $\mathcal{K}$ are at least $\sqrt{\varepsilon}$ apart from each other.\footnote{The filtering step involves removing elements within a circle centered on a vector (i.e., eliminating points inside the green circle in \cref{fig:flow2,fig:flow3}). This process evokes the imagery of lotus leaves, which inspired us to name the proposed method ``LotusFilter''.}

Here, the accuracy of the top-1 result (Recall@1) after filtering remains equal to that of the initial search results. This is because the top-1 result from the initial search is always included in $\mathcal{K}$ in \texttt{L4} during the first iteration.

Note that the proposed approach is faster than existing methods for the following intuitive reasons:
\begin{itemize}
\item The filtering step processes candidates sequentially ($\mathcal{O}(S)$) in a fast, greedy manner. Many existing methods determine similar items in $\mathcal{S}$ by calculating distances on the fly, requiring $\mathcal{O}(DS^2)$ for all pairs, even when approximated. In contrast, our approach precomputes distances, eliminating on-the-fly calculations and avoiding pairwise computations altogether.
\item The filtering step does not require the original vectors, making it a pure post-processing step for any ANNS modules. In contrast, many existing methods depend on retaining the original vectors and computing distances during the search. Therefore, they cannot be considered pure post-processing, especially since modern ANNS methods often use compressed versions of the original vectors.
\end{itemize}
In \cref{sec:comp_anal}, we discuss the computational complexity in detail and demonstrate that it is $\mathcal{O}(T+S+KL)$.

\subsection{Memory consumption}
With $L=\frac{1}{N}\sum_{n=1}^N|\mathcal{L}_n|$ being the average length of $\mathcal{L}_n$, the memory consumption of the LotusFilter is $64LN$ [bit] with the naive implementation using 64 bit integers.
It is because, from \cref{alg:divsearch}, the LotusFilter requires only a cutoff table $\{\mathcal{L}_n\}_{n=1}^N$ as an auxiliary data structure. 

This result demonstrates that the memory consumption of our proposed LotusFilter can be accurately estimated in advance. We will later show in \cref{tbl:comp} that, for $N=9\times 10^5$, the memory consumption is $1.14 \times 10^9$ [bit] $=$ $136$ [MiB].

\subsection{Theoretical guarantees on diversity}
\label{sec:theory}

For the results obtained by \cref{alg:divsearch}, the diversity term (second term) of the objective function \cref{eq:f} is bounded by $-\varepsilon$ as follows. We construct the final results of \cref{alg:divsearch}, $\mathcal{K}$, by adding an element one by one in \texttt{L4}. For each loop, given a new $k$ in \texttt{L4}, all items whose squared distance to $k$ is less than $\varepsilon$ must be contained in $\mathcal{L}_k$. Such close items are removed from the candidates $\mathcal{S}$ in \texttt{L6}. Thus, for all $i, j \in \mathcal{K}$ where $i \ne j$, 
$\Vert \mathbf{x}_i - \mathbf{x}_j \Vert_2^2 \ge \varepsilon$ holds, resulting in $ - \min_{i, j \in \mathcal{K}, i \ne j} \Vert \mathbf{x}_i - \mathbf{x}_j \Vert_2^2 \le - \varepsilon$.

This result shows that the proposed LotusFilter can always ensure diversity, where we can adjust the degree of diversity using the parameter $\varepsilon$.

\subsection{Safeguard against over-pruning}
Filtering can sometimes prune too many candidates from $\mathcal{S}$. To address this issue, a safeguard mode is available as an option. Specifically, if $\mathcal{L}_k$ in \texttt{L6} is large and $|\mathcal{S}|$ drops to zero, no further elements can be popped. If this occurs, $\mathcal{K}$ returned by \cref{alg:divsearch} may have fewer elements than $K$.

With the safeguard mode activated, the process will terminate immediately when excessive pruning happens in \texttt{L6}. The remaining elements in $\mathcal{S}$ will be added to $\mathcal{K}$. This safeguard ensures that the final result meets the condition $|\mathcal{K}| = K$. In this scenario and only in this scenario, the theoretical result discussed in \cref{sec:theory} does not hold.

\section{Complexity Analysis} 
\label{sec:comp_anal}

We prove that the computational complexity of \cref{alg:divsearch} is $\mathcal{O}(T+S+KL)$ on average. This is fast because just accessing the used variables requires the same cost.

The filtering step of our LotusFilter (\texttt{L3-L6} in \cref{alg:divsearch}) is quite simple, but it is unclear whether it can be executed efficiently. Specifically, for $\mathcal{S}$, \texttt{L4} requires a pop operation, and \texttt{L6} removes an element. These two operations cannot be efficiently implemented with basic data structures like arrays, sets, or priority queues.

To address this, we introduce a data structure called OrderedSet. While OrderedSet has a higher memory consumption, it combines the properties of both a set and an array. We demonstrate that by using OrderedSet, the operations in the while loop at \texttt{L3} can be run in $O(L)$.

\subsection{Main result}

\begin{proposition}
The computational complexity of the search and filter algorithm in \cref{alg:divsearch} is $\mathcal{O}(T+S+KL)$ on average using the OrderedSet data structure for $\mathcal{S}$.
\end{proposition}

\begin{proof}
In \texttt{L1}, the search takes $\mathcal{O}(T)$, and the initialization of $\mathcal{S}$ takes $\mathcal{O}(S)$. The loop in \texttt{L3} is executed at most $K$ times. Here, the cost inside the loop is $\mathcal{O}(L)$. That is, \textsc{Pop} on $\mathcal{S}$ takes $\mathcal{O}(L)$ in \texttt{L4}. Adding an element to a set takes $\mathcal{O}(1)$ in \texttt{L5}. The $L$ times deletion for $\mathcal{S}$ in \texttt{L6} takes $\mathcal{O}(L)$. In total, the computational cost is $\mathcal{O}(T+S+KL)$.
\end{proof}

To achieve the above, we introduce the data structure called OrderedSet to represent $\mathcal{S}$. An OrderedSet satisfies $\mathcal{O}(S)$ for initialization, $\mathcal{O}(L)$ for \textsc{Pop}, and $\mathcal{O}(1)$ for the deletion of a single item.

\subsection{OrderedSet}

OrderedSet, as its name suggests, is a data structure representing a set while maintaining the order of the input array. OrderedSet combines the best aspects of arrays and sets at the expense of memory consumption. See the swift-collections package\footnote{\url{https://swiftpackageindex.com/apple/swift-collections/1.1.0/documentation/orderedcollections/orderedset}} in the Swift language for the reference implementation. We have found that this data structure implements the \textsc{Pop} operation in $\mathcal{O}(L)$.

For a detailed discussion of the implementation, hereafter, we consider the input to OrderedSet as an array $\mathbf{v} = [v[1], v[2], \dots, v[V]]$ with $V$ elements (i.e., the input to $\mathcal{S}$ in \texttt{L1} of \cref{alg:divsearch} is an array of integers).

\paragraph{Initialization:}
We show that the initialization of OrderedSet takes $\mathcal{O}(V)$. OrderedSet takes an array $\mathbf{v}$ of length $V$ and converts it into a set (hash table) $\mathcal{V}$:
\begin{equation}
\mathcal{V} \gets \textsc{Set}(\mathbf{v}).
\end{equation}
This construction takes $\mathcal{O}(V)$. Then, a counter $c \in \{1, \dots, V\}$ indicating the head position is prepared and initialized to $c \gets 1$. The OrderedSet is a simple data structure that holds $\mathbf{v}$, $\mathcal{V}$, and $c$.
OrderedSet has high memory consumption because it retains both the original array $\mathbf{v}$ and its set representation $\mathcal{V}$.
An element in $\mathcal{V}$ must be accessed and deleted in constant time on average. We utilize a fast open-addressing hash table \texttt{boost::unordered\_flat\_set} in our implementation\footnote{\url{https://www.boost.org/doc/libs/master/libs/unordered/doc/html/unordered/intro.html}}.
In \texttt{L1} of \cref{alg:divsearch}, this initialization takes $\mathcal{O}(S)$.

\paragraph{Remove:}
The operation to remove an element $a$ from OrderedSet is implemented as follows with an average time complexity of $\mathcal{O}(1)$: 
\begin{equation}
    \mathcal{V} \gets \mathcal{V} \setminus \{a\}.
\end{equation}
In other words, the element is deleted only from $\mathcal{V}$. As the element in $\mathbf{v}$ remains, the deletion is considered shallow.
In \texttt{L6} of \cref{alg:divsearch}, the $L$ removals result in an $\mathcal{O}(L)$ cost.

\paragraph{Pop:}
Finally, the \textsc{Pop} operation, which removes the first element, is realized in $\mathcal{O}(\Delta)$ as follows:
\begin{itemize}
    \item Step 1: Repeat $c \gets c+1$ until $v[c] \in \mathcal{V}$
    \item Step 2: $\mathcal{V} \gets \mathcal{V} \setminus \{v[c]\}$
    \item Step 3: Return $v[c]$
    \item Step 4: $c \gets c+1$
\end{itemize}
Step 1 moves the counter until a valid element is found. Here, the previous head (or subsequent) elements might have been removed after the last call to \textsc{Pop}. In such cases, the counter must move along the array until it finds a valid element. Let $\Delta$ be the number of such moves; this counter update takes $\mathcal{O}(\Delta)$. In Step 2, the element is removed in $\mathcal{O}(1)$ on average. In Step 3, the removed element is returned, completing the \textsc{Pop} operation. Step 4 updates the counter position accordingly.

Thus, the total time complexity is $\mathcal{O}(\Delta)$. Here, $\Delta$ represents the ``number of consecutively removed elements from the previous head position since the last call to \textsc{Pop}''. In our problem setting, between two calls to \textsc{Pop}, at most $L$ elements can be removed (refer to \texttt{L6} in \cref{alg:divsearch}). Thus,
\begin{equation}
\Delta \le L.
\end{equation}
Therefore, the \textsc{Pop} operation is $\mathcal{O}(L)$ in \cref{alg:divsearch}.

Using other data structures, achieving both \textsc{Pop} and \textsc{Remove} operations efficiently is challenging. With an array, \textsc{Pop} can be accomplished in $ \mathcal{O}(\Delta)$ in the same way. However, removing a specific element requires a linear search, which incurs a cost of $\mathcal{O}(V)$. On the other hand, if we use a set (hash table), deletion can be done in $\mathcal{O}(1)$, but \textsc{Pop} cannot be implemented. Please refer to the supplemental material for a more detailed comparison of data structures.

\section{Training} 
\label{sec:opt}

The proposed method intuitively realizes diverse search by removing similar items from the search results, but it is unclear how it contributes explicitly to the objective function \cref{eq:dnn}. Here, by learning the threshold $\varepsilon$ in advance, we ensure that our LotusFilter effectively reduces \cref{eq:dnn}.

First, let's confirm the parameters used in our approach; $\lambda, S, K,$ and $\varepsilon$.
Here, $\lambda$ is set by the user to balance the priority between search and diversification. $K$ is the number of final search results and must also be set by the user. $S$ governs the accuracy and speed of the initial search. Setting $S$ is not straightforward, but it can be determined based on runtime requirements, such as setting $S=3K$. The parameter $\varepsilon$ is less intuitive; a larger $\varepsilon$ increases the cutoff table size $L$, impacting both results and runtime. The user should set $\varepsilon$ minimizing $f$, but this setting is not straightforward.

To find the optimal $\varepsilon$, we rewrite the equations as follows. First, since $\mathcal{S}$ is the search result of $\mathbf{q}$, we can write $\mathcal{S} = \mathrm{NN}(\mathbf{q},~S)$. Here, we explicitly express the solution $f^*$ of \cref{eq:dnn} as a function of $\varepsilon$ and $\mathbf{q}$ as follows.

\begin{equation}
    f^*(\varepsilon, \mathbf{q}) = 
    \argmin_{\mathcal{K} \subseteq \mathrm{NN}(\mathbf{q},~S), ~ |\mathcal{K}|
    = K} f(\mathcal{K}).
\end{equation}
We would like to find $\varepsilon$ that minimizes the above.
Since $\mathbf{q}$ is a query data provided during the search phase, we cannot know it beforehand. Therefore, we prepare training query data $\mathcal{Q}_\mathrm{train} \subset \mathbb{R}^D$ in the training phase. This training query data can usually be easily prepared using a portion of the database vectors. Assuming that this training query data is drawn from a distribution similar to the test query data, we solve the following.
\begin{equation}
    \varepsilon^* = \argmin_{\varepsilon} \underset{\mathbf{q} \in \mathcal{Q}_\mathrm{train}}{\mathbb{E}} \left [ f^*(\varepsilon, \mathbf{q}) \right ].
    \label{eq:opt}
    \end{equation}

This problem is a nonlinear optimization for a single variable without available gradients. One could apply a black-box optimization~\cite{kdd_akiba2019} to solve this problem, but we use a more straightforward approach, bracketing~\cite{book_kochenderfer2019}, which recursively narrows the range of the variable. See the supplementary material for details. This simple method achieves sufficient accuracy as shown later in \cref{fig:epsilon_f}.

\section{Evaluation}
\label{sec:exp}

\begin{table*}[tb]
    \centering
    \begin{tabular}{@{}lllllllll@{}} \toprule
            & \multicolumn{3}{c}{Cost function ($\downarrow$)} & \multicolumn{3}{c}{Runtime [ms/query] ($\downarrow$)} & \multicolumn{2}{c}{Memory overhead [bit] ($\downarrow$)} \\ \cmidrule(l){2-4} \cmidrule(l){5-7} \cmidrule(l){8-9}
        Filtering & Search & Diversification & Final ($f$)& Search & Filter & Total & $\{ \mathbf{x}_n \}_{n=1}^N$ & $\{ \mathcal{L}_n \}_{k=1}^K$ \\ \midrule
        None (Search only)     & $0.331$ & $-0.107$ & $0.200$ & $0.855$ & -       & $\mathbf{0.855}$ & -                    & - \\ 
        Clustering             & $0.384$ & $-0.152$ & $0.223$ & $0.941$ & $6.94$  & $7.88$  & $4.42 \times {10}^{10}$ & - \\      
        GMM~\cite{or_ravi1994} & $0.403$ & $-0.351$ & \underline{$0.177$} & $0.977$ & $13.4$  & $14.4$  & $4.42 \times {10}^{10}$ & - \\      
        LotusFilter (Proposed) & $0.358$ & $-0.266$ & $\mathbf{0.171}$ & $1.00$ & $0.02$ & \underline{$1.03$}  & -                    & $1.14 \times {10}^9$ \\ \bottomrule
    \end{tabular}
    \caption{Comparison with existing methods for the OpenAI dataset. The parameters are $\lambda=0.3,~K=100,~S=500,~\varepsilon^*=0.277,$ and $L=19.8$. The search step is with HNSW~\cite{tpami_malkov2020}. Bold and underlined scores represent the best and second-best results, respectively.}
    \label{tbl:comp}
\end{table*}

In this section, we evaluate the proposed LotusFilter. All experiments were conducted on an AWS EC2 \texttt{c7i.8xlarge} instance (3.2GHz  Intel Xeon CPU, 32 virtual cores, 64GiB memory). We ran preprocessing using multiple threads while the search was executed using a single thread. For ANNS, we used HNSW~\cite{tpami_malkov2020} from the faiss library~\cite{arXiv_douze2024}. The parameters of HNSW were \texttt{efConstruction=40}, \texttt{efSearch=16}, and \texttt{M=256}. LotusFilter is implemented in C++17 and called from Python using nanobind~\cite{web_nanobind}. Our code is publicly available at \url{https://github.com/matsui528/lotf}.

We utilized the following datasets:

\begin{itemize}
\item OpenAI Dataset~\cite{arXiv_simhadri2024, arXiv_oguri2024}: This dataset comprises 1536-dimensional text features extracted from WikiText using OpenAI's text embedding model. It consists of 900,000 base vectors and 100,000 query vectors. We use this dataset for evaluation, considering that the proposed method is intended for application in RAG systems.
\item MS MARCO Dataset~\cite{arxiv_bajaj2016}: This dataset includes Bing search logs. We extracted passages from the v1.1 validation set, deriving 768-dimensional BERT features~\cite{naacl_devlin2019}, resulting in 38,438 base vectors and 1,000 query vectors. We used this dataset to illustrate redundant texts.
\item Revisited Paris Dataset~\cite{cvpr_radenovic2018}: This image dataset features landmarks in Paris, utilizing 2048-dimensional R-GeM~\cite{tpami_radenovic2018} features with 6,322 base and 70 query vectors. It serves as an example of data with many similar images.
\end{itemize}
We used the first 1,000 vectors from base vectors for hyperparameter training ($\mathcal{Q}_\mathrm{train}$ in \cref{eq:opt}).

\subsection{Comparison with existing methods}
\paragraph{Existing methods}
We compare our methods with existing methods in \cref{tbl:comp}. The existing methods are the ANNS alone (i.e., HNSW only), clustering, and the GMM~\cite{or_ravi1994,aaai_amagata2023}.
\begin{itemize}
\item ANNS alone (no filtering): An initial search is performed to obtain $K$ results. We directly use them as the output.
\item Clustering: After obtaining the initial search result $\mathcal{S}$, we cluster the vectors $\{\mathbf{x}_s\}_{s \in \mathcal{S}}$ into $K$ groups using k-means clustering. The nearest neighbors of each centroid form the final result $\mathcal{K}$. Clustering serves as a straightforward approach to diversifying the initial search results with the running cost of $\mathcal{O}(DKS)$. To perform clustering, we require the original vectors $\{\mathbf{x}_n\}_{n=1}^N$.
\item GMM: GMM is a representative approach for extracting a diverse subset from a set. After obtaining the initial search result $\mathcal{S}$, we iteratively add elements to $\mathcal{K}$ according to \( j^* = \arg\max_{j \in \mathcal{S} \setminus \mathcal{K}} \left( \min_{i \in \mathcal{K}} \Vert \mathbf{x}_i - \mathbf{x}_j \Vert_2^2 \right) \), updating $\mathcal{K}$ as \(\mathcal{K} \gets \mathcal{K} \cup \{j^*\}\) in each step. This GMM approach produces the most diverse results from the set $\mathcal{S}$. With a bit of refinement, GMM can be computed in \(\mathcal{O}(DKS)\). Like k-means clustering, GMM also requires access to the original vectors \(\{\mathbf{x}_n\}_{n=1}^N\).
\end{itemize}
We consider the scenario of obtaining $\mathcal{S}$ using modern ANNS methods like HNSW, followed by diversification. Since no existing methods can be directly compared in this context, we use simple clustering and GMM as baselines.

Well-known DNNS methods, like Maximal Marginal Relevance (MMR)~\cite{sigir_carbonell1998}, are excluded from comparison due to their inability to directly utilize ANNS, resulting in slow performance. Directly solving \cref{eq:dnn} is also excluded because of its high computational cost. Note that MMR can be applied to $\mathcal{S}$ rather than the entire database vectors. This approach is similar to the GMM described above and can be considered an extension that takes the distance to the query into account. Although it has a similar runtime as GMM, its score was lower, so we reported the GMM score.

In the ``Cost function'' of \cref{tbl:comp}, the ``Search'' refers to the first term in \cref{eq:f}, and the ``Diversification'' refers to the second term. The threshold $\varepsilon$ is the value obtained from \cref{eq:opt}. The runtime is the average of three trials.

\paragraph{Results}
From \cref{tbl:comp}, we observe the following results:
\begin{itemize}
\item In the case of NN search only, it is obviously the fastest; however, the results are the least diverse (with a diversification term of $-0.107$).
\item Clustering is simple but not promising. The final score is the worst ($f=0.223$), and it takes 10 times longer than search-only ($7.88$ [ms/query]).
\item GMM achieves the most diverse results ($-0.351$), attaining the second-highest final performance ($f=0.177$). However, GMM is slow ($14.4$ [ms/query]), requiring approximately 17 times the runtime of search-only.
\item The proposed LotusFilter achieves the highest performance ($f=0.171$). It is also sufficiently fast ($1.03$ [ms/query]), with the filtering step taking only $0.02$ [ms/query]. As a result, it requires only about 1.2 times the runtime of search-only.
\item Clustering and GMM consume 40 times more memory than LotusFilter. Clustering and GMM require the original vectors, costing $32ND$ [bits] using 32-bit floating-points, which becomes especially large for datasets with a high $D$. In contrast, the memory cost of the proposed method is $64LN$ using 64-bit integers.
\end{itemize}
The proposed method is an effective filtering approach regarding performance, runtime, and memory efficiency, especially for high-dimensional vectors. For low-dimensional vectors, simpler baselines may be more effective. Please see the supplemental material for details.

\begin{figure*}[tb]
  \begin{minipage}[]{0.32\linewidth}
    \centering
    \includegraphics[width=1.0\linewidth]{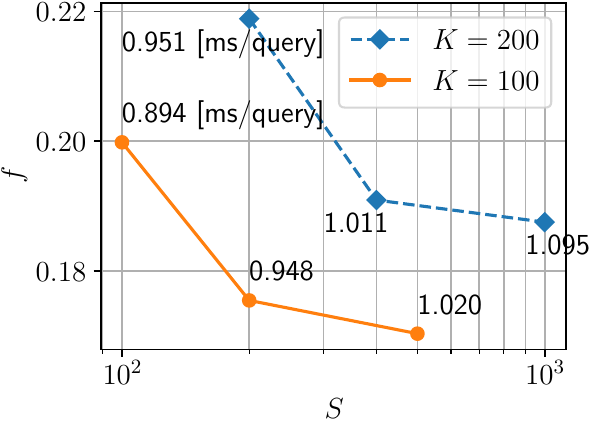}
    \caption{Fix $K$, vary $S$}\label{fig:s_f}
  \end{minipage} \hfill
  \begin{minipage}[]{0.32\linewidth}
    \centering
    \includegraphics[width=1.0\linewidth]{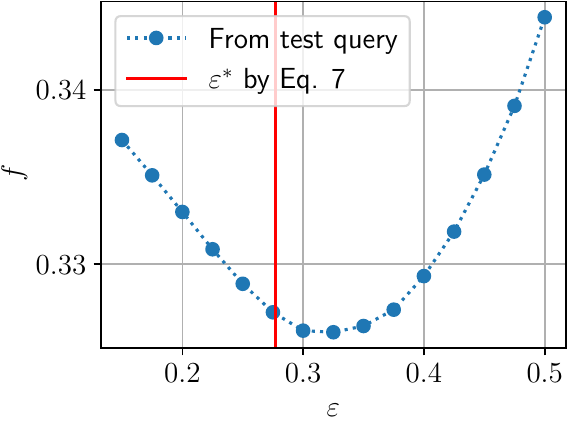}
    \caption{Evaluate $\varepsilon^*$ by \cref{eq:opt}}\label{fig:epsilon_f}
  \end{minipage} \hfill
  \begin{minipage}[]{0.32\linewidth}
    \centering
    \scalebox{0.75}{
    \begin{tabular}{@{}llllll@{}} \toprule
        & & &  & \multicolumn{2}{c}{Runtime [s]} \\ \cmidrule(l){5-6}
        $N$ & $\lambda$ & $\varepsilon^*$ & $L$ & Train & Build \\ \midrule
        $9 \times 10^3$ & 0.3 & 0.39 & 8.7  & 96 & 0.16 \\ 
                        & 0.5 & 0.42 & 19.6 & 99 & 0.17 \\
        $9 \times 10^4$ & 0.3 & 0.33 & 10.1 & 176 & 3.8 \\ 
                        & 0.5 & 0.36 & 23.5 & 177 & 3.9 \\
        $9 \times 10^5$ & 0.3 & 0.27 & 18.4 & 1020 & 54 \\ 
                        & 0.5 & 0.29 & 29.3 & 1087 & 54 \\ \bottomrule
        \end{tabular}
    }
    \captionof{table}{Train and build}\label{tbl:opt}

  \end{minipage}
\end{figure*}

\subsection{Impact of the number of initial search results}
When searching, users are often interested in knowing how to set $S$, the size of the initial search result. We evaluated this behavior for the OpenAI dataset in \cref{fig:s_f}. Here, $\lambda = 0.3$, and $\varepsilon$ is determined by solving \cref{eq:opt} for each point.

Taking more candidates in the initial search (larger $S$) results in the following:
\begin{itemize}
    \item Overall performance improves (lower $f$), as having more candidates is likely to lead to better solutions.
    \item On the other hand, the runtime gradually increases. Thus, there is a clear trade-off in $S$'s choice.
\end{itemize}

\subsection{Effectiveness of training}
We investigated how hyperparameter tuning in the training phase affects final performance using the OpenAI dataset. While simple, we found that the proposed training procedure achieves sufficiently good performance.

The training of $\varepsilon$ as described in \cref{sec:opt} is shown in \cref{fig:epsilon_f} ($\lambda=0.3, K=100, S=500$). Here, the blue dots represent the actual calculation of $f$ using various $\varepsilon$ values with the test queries. The goal is to obtain $\varepsilon$ that achieves the minimum value of this curve in advance using training data. The red line represents the $\varepsilon^*$ obtained from the training queries via \cref{eq:opt}. Although not perfect, we can obtain a reasonable solution. These results demonstrate that the proposed data structure can perform well by learning the parameters in advance using training data.

\subsection{Preprocessing time}
\cref{tbl:opt} shows the training and construction details (building the cutoff table) with $K=100$ and $S=500$ for the OpenAI dataset. Here, we vary the number of database vectors $N$. For each condition, $\varepsilon$ is obtained by solving \cref{eq:opt}. The insights obtained are as follows:
\begin{itemize}
\item As $N$ increases, the time for training and construction increases, and $L$ also becomes larger, whereas $\varepsilon^*$ decreases.
\item As $\lambda$ increases, $\varepsilon^*$ and $L$ increase, and training and construction times slightly increase.
\item $L$ is at most 30 within the scope of this experiment.
\item Training and construction each take a maximum of approximately  1,100 seconds and 1 minute, respectively. This runtime is sufficiently fast but could potentially be further accelerated using specialized hardware like GPUs. 
\end{itemize}

\subsection{Qualitative evaluation for texts}

\begin{table}[tb]

    \centering
    \scalebox{0.8}{
    \begin{tabular}{@{}lp{9.5cm}@{}} \toprule
     \multicolumn{2}{@{}l@{}}{\textbf{Query}: ``Tonsillitis is a throat infection that occurs on the tonsil.''} \\ \midrule
     \multicolumn{2}{@{}l@{}}{\textbf{Results by nearest neighbor search}} \\
     1: & ``Tonsillitis refers to the inflammation of the pharyngeal tonsils and is the primary cause of sore throats.'' \\
     2: & \textcolor{red}{``Strep throat is a bacterial infection in the throat and the tonsils.''} \\
     3: & \textcolor{red}{``Strep throat is a bacterial infection of the throat and tonsils.''} \\
     4: & \textcolor{red}{``Strep throat is a bacterial infection of the throat and tonsils.''} \\
     5: & ``Mastoiditis is an infection of the spaces within the mastoid bone.'' \\ \midrule
     \multicolumn{2}{@{}l@{}}{\textbf{Results by diverse nearest neighbor search (proposed)}} \\
     1: & ``Tonsillitis refers to the inflammation of the pharyngeal tonsils and is the primary cause of sore throats.`` \\
     2: & ``Strep throat is a bacterial infection in the throat and the tonsils.`` \\
     3: & ``Mastoiditis is an infection of the spaces within the mastoid bone.`` \\
     4: & ``Tonsillitis (enlarged red tonsils) is caused by a bacterial (usually strep) or viral infection.`` \\
     5: & ``Spongiotic dermatitis is a usually uncomfortable dermatological condition which most often affects the skin of the chest, abdomen, and buttocks.`` \\ \bottomrule
    \end{tabular}
    }
    \caption{Qualitative evaluation on text data using MS MARCO.}
    \label{tbl:msmarco}
\end{table}

\begin{figure*}[tb]
\centering
\includegraphics[width=1.0\linewidth]{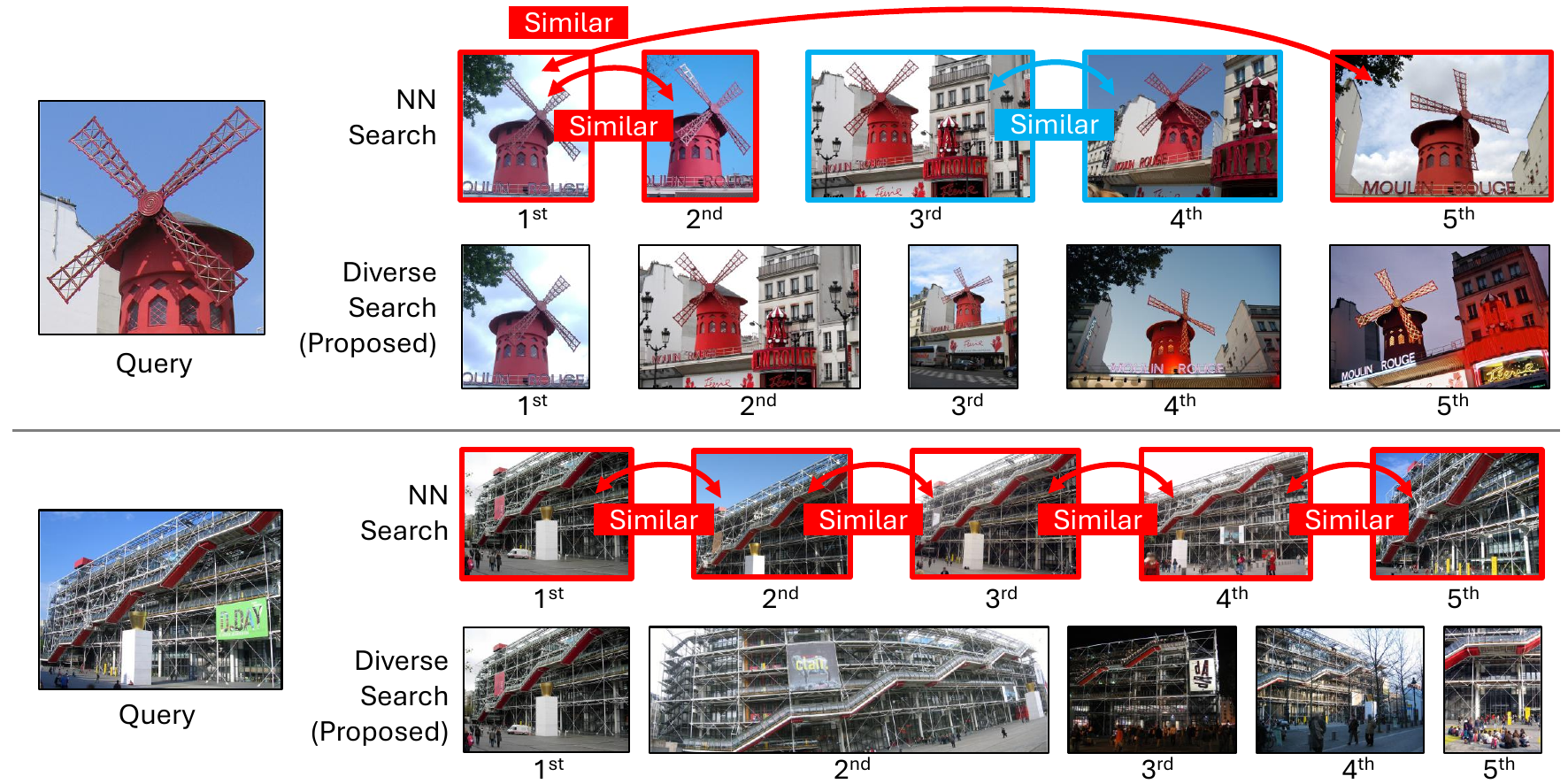}
\caption{Qualitative evaluation on image data using Revisited Paris.}
\label{fig:paris}
\end{figure*}

This section reports qualitative results using the MS MARCO dataset (\cref{tbl:msmarco}). This dataset contains many short, redundant passages, as anticipated for real-world use cases of RAG. We qualitatively compare the results of the NNS and the proposed DNNS on such a redundant dataset. The parameters are $K=10$, $S=50$, $\lambda=0.3$, and $\varepsilon^*=18.5$.

Simple NNS results displayed nearly identical second, third, and fourth-ranked results (highlighted in red), while the proposed LotusFilter eliminates this redundancy. This tendency to retrieve similar data from the scattered dataset is common if we run NNS. Eliminating such redundant results is essential for real-world RAG systems. See the supplemental material for more examples.

The proposed LotusFilter is effective because it obtains diverse results at the data structure level. While engineering solutions can achieve diverse searches, such solutions are complex and often lack runtime guarantees. In contrast, LotusFilter is a simple post-processing module with computational guarantees. This simplicity makes it an advantageous building block for complex systems, especially in applications like RAG.

\subsection{Qualitative evaluation for images}

This section reports qualitative evaluations of images. Here, we consider an image retrieval task using image features extracted from the Revisited Paris dataset (\cref{fig:paris}). The parameters are set to $K=10$, $S=100$, $\lambda=0.5$, and $\varepsilon^*=1.14$.

In the first example, a windmill image is used as a query to find similar images in the dataset. The NNS results are shown in the upper row, while the proposed diverse search results are in the lower row. The NNS retrieves images close to the query, but the first, second, and fifth images show windmills from similar angles, with the third and fourth images differing only in sky color. In a recommendation system, such nearly identical results would be undesirable. The proposed diverse search, however, provides more varied results related to the query.

In the second example, the query image is a photograph of the Pompidou Center taken from a specific direction. In this case, all the images retrieved by the NNS have almost identical compositions. However, the proposed approach can retrieve images captured from various angles.

It is important to note that the proposed LotusFilter is simply a post-processing module, which can be easily removed. For example, if the diverse search results are less appealing, simply deactivating LotusFilter would yield the standard search results. Achieving diverse search through engineering alone would make it more difficult to switch between results in this way.

\subsection{Limitations and future works}

The limitations and future works are as follows:
\begin{itemize}
\item LotusFilter involves preprocessing steps. Specifically, we optimize $\varepsilon$ for parameter tuning, and a cutoff table needs to be constructed in advance.
\item During $\varepsilon$ learning, $K$ needs to be determined in advance. In practical applications, there are many cases where $K$ needs to be varied. If $K$ is changed during the search, it is uncertain whether $\varepsilon^*$ is optimal.
\item A theoretical bound has been established for the diversification term in the cost function; however, there is no theoretical guarantee for the total cost.
\item Unlike ANNS alone, LotusFilter requires additional memory for a cutoff table. Although the memory usage is predictable at $64LN$ [bits], it can be considerable, especially for large values of $N$.
\item When $D$ is small, more straightforward methods (such as GMM) may be the better option.
\item The proposed method determines a global threshold $\varepsilon$. Such a single threshold may not work well for challenging datasets.
\item The end-to-end evaluation of the RAG system is planned for future work. Currently, the accuracy is only assessed by \cref{eq:f}, and the overall performance within the RAG system remains unmeasured. A key future direction is employing LLM-as-a-judge to evaluate search result diversity comprehensively.
\end{itemize}

\section{Conclusions}
\label{sec:conclusion}

We introduced the LotusFilter, a fast post-processing module for DNNS. The method entails creating and using a cutoff table for pruning. Our experiments showed that this approach achieves diverse searches in a similar time frame to the most recent ANNS.

\section*{Acknowledgement}
We thank Daichi Amagata and Hiroyuki Deguchi for reviewing this paper, and we appreciate Naoki Yoshinaga for providing the inspiration for this work.

{
    \small
    \bibliographystyle{ieeenat_fullname}
    \bibliography{main}
}

\clearpage
\setcounter{page}{1}
\maketitlesupplementary

\newcommand\beginsupplement{%
        \setcounter{table}{0}
        \renewcommand{\thetable}{\Alph{table}}%
        \setcounter{figure}{0}
        \renewcommand{\thefigure}{\Alph{figure}}%
        \setcounter{equation}{0}
        \renewcommand{\theequation}{\Alph{equation}}%
        \setcounter{section}{0}
        \renewcommand{\thesection}{\Alph{section}}%
        \setcounter{algocf}{0} 
        \renewcommand{\thealgocf}{\Alph{algocf}} 
     }
\beginsupplement

\section{Selection of data structures}

\begin{table*}[tb]
    \centering
    \begin{tabular}{@{}lllll@{}} \toprule
        Method & \textsc{Pop}() &  \textsc{Remove}($a$) & Constant factor & Overall complexity of Algorithm 2\\ \midrule
        Array & $\mathcal{O}(\Delta)$  & $\mathcal{O}(V)$ & & $\mathcal{O}(T+KLS)$\\
        Set (hash table) & - &  $\mathcal{O}(1)$ & & N/A\\
        List & $\mathcal{O}(1)$ & $\mathcal{O}(V)$ & & $\mathcal{O}(T+KLS)$\\ 
        Priority queue & $\mathcal{O}(\log V)$  & $\mathcal{O}(V)$ & & $\mathcal{O}(T+KLS)$\\
        List + dictionary (hash table) & $\mathcal{O}(1)$ & $\mathcal{O}(1)$ & Large & $\mathcal{O}(T+S+KL)$\\
        OrderedSet: array + set (hash table) & $\mathcal{O}(\Delta)$ & $\mathcal{O}(1)$ & & $\mathcal{O}(T+S+KL)$\\ \bottomrule
    \end{tabular}
    \caption{The average computational complexity to achieve operations on $\mathcal{S}$}
    \label{tbl:s}
\end{table*}

We introduce alternative data structures for $\mathcal{S}$ and demonstrate that the proposed OrderedSet is superior. As introduced in Sec. 5.2, an input array $\mathbf{v} = [v[1], v[2], \dots, v[V]]$ containing $V$ elements is given. The goal is to realize a data structure that efficiently performs the following operations:

\begin{itemize}
\item \textsc{Pop}: Retrieve and remove the foremost element while preserving the order of the input array.
\item \textsc{Remove}: Given an element as input, delete it from the data structure.
\end{itemize}

The average computational complexity of these operations for various data structures, including arrays, sets, priority queues, lists, and their combinations, are summarized in \cref{tbl:s}.

\paragraph{Array} When using an array directly, the \textsc{Pop} operation follows the same procedure as OrderedSet. However, element removal incurs a cost of $\mathcal{O}(V)$. This removal is implemented by performing a linear search and marking the element with a tombstone. Due to the inefficiency of this removal process, arrays are not a viable option.

\paragraph{Set}
If we convert the input array into a set (e.g., \texttt{std::unordered\_set} in C++ or \texttt{set} in Python), element removal can be achieved in $\mathcal{O}(1)$. However, since the set does not maintain element order, we cannot perform the \textsc{Pop} operation, making this approach unsuitable.

\paragraph{List}
Consider converting the input array into a list (e.g., a doubly linked list such as \texttt{std::list} in C++). The first position in the list is always accessible, and removal from this position is straightforward, so \textsc{Pop} can be executed in $\mathcal{O}(1)$. However, for \textsc{Remove}, a linear search is required to locate the element, resulting in a cost of $\mathcal{O}(V)$. Hence, this approach is slow.

\paragraph{Priority queue}
A priority queue is a commonly used data structure for implementing \textsc{Pop}. C++ STL has a standard implementation such as \texttt{std::priority\_queue}. If the input array is converted into a priority queue, the \textsc{Pop} operation can be performed in $\mathcal{O}(\log V)$. However, priority queues are not well-suited for removing a specified element, as this operation requires a costly full traversal in a naive implementation. Thus, priority queues are inefficient for this purpose.

\paragraph{List + dictionary}
Combining a list with a dictionary (hash table) achieves both \textsc{Pop} and \textsc{Remove} operations in $\mathcal{O}(1)$, making it the fastest from a computational complexity perspective. The two data structures, a list and a dictionary, are created in the construction step. First, the input array is converted into a list to maintain order. Next, the dictionary is created with a key corresponding to an element in the array, and a value is a pointer pointing to a corresponding node in the list.

During removal, the element is removed from the dictionary, and the corresponding node in the list is also removed. This node removal is possible since we know its address from the dictionary. This operation ensures the list maintains the order of remaining elements. For \textsc{Pop}, the first element in the list is extracted and removed, and the corresponding element in the dictionary is also removed.

While the list + dictionary combination achieves the best complexity, its constant factors are significant. Constructing two data structures during initialization is costly. \textsc{Remove} must also remove elements from both data structures. Furthermore, in our target problem (Algorithm 2), the cost of set deletions within the for-loop (\texttt{L6}) is already $\mathcal{O}(L)$. Thus, even though \textsc{Pop} is $\mathcal{O}(1)$, it does not improve the overall computational complexity. Considering these factors, we opted for OrderedSet.

\paragraph{OrderedSet}
As summarized in \cref{tbl:s}, our OrderedSet introduced in Sec. 5.2 combines the advantages of arrays and hash tables. During initialization, only a set is constructed. The only operation required for removals is deletion from the set, resulting in smaller constant factors than other methods.

\section{Details of training}

\begin{algorithm}[tb]
\DontPrintSemicolon
\KwIn{$\mathcal{Q}_\mathrm{train} \subset \mathbb{R}^D$}
\KwHyperParam{$\varepsilon_\mathrm{max} \in \mathbb{R}$, ~ $W \in \mathbb{R}$, ~ $\mathcal{I}$, ~  $\lambda \in [0, 1]$, ~ $S$, ~ $K$}
\KwOut{$\varepsilon^* \in \mathbb{R}$}
$\varepsilon_\mathrm{left} \gets 0$ \PythonStyleComment*[r]{Lower bound}
$\varepsilon_\mathrm{right} \gets \varepsilon_\mathrm{max}$ \PythonStyleComment*[r]{Upper bound}
$r \gets \varepsilon_\mathrm{right} - \varepsilon_\mathrm{left}$ \PythonStyleComment*[r]{Search range}
$\varepsilon^* \gets \infty$ \;
$f^* \gets \infty$ \;
\SimpleRepeat{$\mathrm{5~times}$}{
    \PythonStyleComment*[l]{Sampling $W$ candidates at equal intervals from the search range}
    $\mathcal{E} \gets \left \{\varepsilon_\mathrm{left} + i \frac{\varepsilon_\mathrm{left} - \varepsilon_\mathrm{right}}{W} \mid i \in \{0, \dots, W\} \right \}$ \;
    \PythonStyleComment*[l]{Evaluate all candidates and find the best one}
    \For{$\varepsilon \in \mathcal{E}$} {
        $f \gets \underset{\mathbf{q} \in \mathcal{Q}_\mathrm{train}}{\mathbb{E}} \left [ f^*(\varepsilon, \mathbf{q}) \right ]$ \;
        \If{$f < f^*$}{
            $\varepsilon^* \gets  \varepsilon$ \;
            $f^* \gets f$ \;
        }
    }
    $r \gets r / 2$ \PythonStyleComment*[r]{Shrink the range}
    \PythonStyleComment*[l]{Update the bounds}
    $\varepsilon_\mathrm{left} \gets \max (\varepsilon^* - r, ~ 0)$ \;
    $\varepsilon_\mathrm{right} \gets \min (\varepsilon^* + r, ~ \varepsilon_\mathrm{max})$ \;
}
\Return{$\varepsilon^*$}

\caption{Training for $\varepsilon$}
\label{alg:train}
\end{algorithm}

In \cref{alg:train}, we describe our training approach for $\varepsilon$ (Eq. (7)) in detail. The input to the algorithm is the training query vectors $\mathcal{Q}_\mathrm{train} \subset \mathbb{R}^D$, which can be prepared by using a part of the database vectors. The output is $\varepsilon^*$ that minimizes the evaluation function $f^*(\varepsilon, \mathbf{q})$ defined in Eq. (6). Since this problem is a non-linear single-variable optimization, we can apply black-box optimization~\cite{optuna_2019}, but we use a more straightforward approach, bracketing~\cite{book_kochenderfer2019}.

\cref{alg:train} requires several hyperparameters:
\begin{itemize}
\item $\varepsilon_\mathrm{max} \in \mathbb{R}$: The maximum range of $\varepsilon$. This value can be estimated by calculating the inter-data distances for sampled vectors.
\item $W \in \mathbb{R}$: The number of search range divisions. We will discuss this in detail later.
\item LotusFilter parameters: These include $\mathcal{I}$, $\lambda \in [0, 1]$, $S$, and $K$. Notably, this training algorithm fixes $\lambda$, $S$, and $K$, and optimizes $\varepsilon$ under these conditions.
\end{itemize}

First, the search range for the variable $\varepsilon$ is initialized in \texttt{L1} and \texttt{L2}. Specifically, we consider the range $[\varepsilon_\mathrm{left}, \varepsilon_\mathrm{right}]$ and examine the variable within this range $\varepsilon \in [\varepsilon_\mathrm{left}, \varepsilon_\mathrm{right}]$. The size of this range is recorded in \texttt{L3}. The optimization loop is executed in \texttt{L6}, where we decided to perform five iterations. In \texttt{L7}, $W$ candidates are sampled at equal intervals from the current search range of the variable. We evaluate all candidates in \texttt{L8-12}, selecting the best one. Subsequently, the search range is narrowed in \texttt{L13-15}. Here, the size of the search range is gradually reduced in \texttt{L13}. The search range for the next iteration is determined by centering around the current optimal value $\varepsilon^*$ and extending $r$ in both directions (\texttt{L14-15}).

The parameter $W$ is not a simple constant but is dynamically scheduled. $W$ is set to 10 for the first four iterations to enable coarse exploration over a wide range. In the final iteration, $W$ is increased to 100 to allow fine-tuned adjustments after the search range has been adequately narrowed.

The proposed training method adopts a strategy similar to beam search, permitting some breadth in the candidate pool while greedily narrowing the range recursively. This approach avoids the complex advanced machine learning algorithms, making it simple and fast (as shown in Table 2, the maximum training time observed in our experiments was less than approximately 1100 seconds on CPUs). As illustrated in Fig. 4, this training approach successfully identifies an almost optimal parameter.

\section{Experiments on memory-efficient datasets}

\begin{table*}[tb]
    \centering
    \begin{tabular}{@{}lllllllll@{}} \toprule
            & \multicolumn{3}{c}{Cost function ($\downarrow$)} & \multicolumn{3}{c}{Runtime [ms/query] ($\downarrow$)} & \multicolumn{2}{c}{Memory overhead [bit] ($\downarrow$)} \\ \cmidrule(l){2-4} \cmidrule(l){5-7} \cmidrule(l){8-9}
        Filtering & Search & Diversification & Final ($f$)& Search & Filter & Total & $\{ \mathbf{x}_n \}_{n=1}^N$ & $\{ \mathcal{L}_n \}_{k=1}^K$ \\ \midrule
        None (Search only)     & $10197$ & $-778$  & $6904$ & $0.241$ & -       & $\mathbf{0.241}$ & -                    & - \\ 
        Clustering             & $11384$ & $-2049$ & $7354$ & $0.309$ & $0.372$  & $0.681$  & $8 \times {10}^{9}$ & - \\      
        GMM~\cite{or_ravi1994} & $12054$ & $-9525$ & $\mathbf{5580}$ & $0.310$ & $0.367$  & $0.677$  & $8 \times 10^{9}$ & - \\      
        LotusFilter (Proposed) & $10648$ & $-5592$ & \underline{$5776$} & $0.310$ & $0.016$ & \underline{$0.326$}  & -                    & $3.7 \times {10}^{10}$ \\ \bottomrule
    \end{tabular}
    \caption{Comparison with existing methods for the MS SpaceV 1M dataset. The parameters are $\lambda=0.3,~K=100,~S=300,~\varepsilon^*=5869,$ and $L=58.3$. The search step is with HNSW~\cite{tpami_malkov2020}. Bold and underlined scores represent the best and second-best results, respectively.}
    \label{tbl:comp_msspacev}
\end{table*}

We present the experimental results on a memory-efficient dataset and demonstrate that simple baselines can be viable choices. Here, we use the Microsoft SpaceV 1M dataset~\cite{pmlr_simhadrik2022}. This dataset consists of web documents represented by features extracted using the Microsoft SpaceV Superion model~\cite{bigdata_shan2021}. While the original dataset contains $N=10^9$ vectors, we used the first $10^7$ vectors for our experiments. We utilized the first $10^3$ entries from the query set for query data. The dimensionality of the vectors is $100$, which is relatively low-dimensional, and each element is represented as an 8-bit integer. Therefore, compared to features like those from CLIP, which are represented in float and often exceed $1000$ dimensions, this dataset is significantly more memory-efficient.

\cref{tbl:comp_msspacev} shows the results. While the overall trends are similar to those observed with the OpenAI dataset in Table~1, there are key differences:
\begin{itemize}
\item LotusFilter remains faster than Clustering and GMM, but the runtime advantage is minor. This result is because LotusFilter's performance does not depend on $D$, whereas Clustering and GMM are $D$-dependent, and thus their performance improves relatively as $D$ decreases.
\item Memory usage is higher for LotusFilter. This is due to the dataset being represented as memory-efficient 8-bit integers, causing the cutoff table of LotusFilter to consume more memory in comparison.
\end{itemize}
From the above, simple methods, particularly GMM, are also suitable for memory-efficient datasets.

\section{Additional results of qualitative evaluation on texts}

\begin{table*}[tb]

    \centering
    \scalebox{0.8}{
    \begin{tabular}{@{}lp{20cm}@{}} \toprule
     \multicolumn{2}{@{}l@{}}{\textbf{Query}: ``This condition is usually caused by bacteria entering the bloodstream and infecting the heart.''} \\ \midrule
     \multicolumn{2}{@{}l@{}}{\textbf{Results by nearest neighbor search}} \\
     1: & \textcolor{red}{``It is a common symptom of coronary heart disease, which occurs when vessels that carry blood to the heart become narrowed and blocked due to atherosclerosis.''} \\
     2: & \textcolor{red}{``It is a common symptom of coronary heart disease, which occurs when vessels that carry blood to the heart become narrowed and blocked due to atherosclerosis.''} \\
     3: & \textcolor{red}{``It is a common symptom of coronary heart disease, which occurs when vessels that carry blood to the heart become narrowed and blocked due to atherosclerosis.''} \\
     4: & ``Cardiovascular disease is the result of the build-up of plaques in the blood vessels and heart.'' \\
     5: & ``The most common cause of myocarditis is infection of the heart muscle by a virus.'' \\ \midrule
     \multicolumn{2}{@{}l@{}}{\textbf{Results by diverse nearest neighbor search (proposed)}} \\
     1: & ``It is a common symptom of coronary heart disease, which occurs when vessels that carry blood to the heart become narrowed and blocked due to atherosclerosis.'' \\
     2: & ``Cardiovascular disease is the result of the build-up of plaques in the blood vessels and heart.'' \\
     3: & ``The most common cause of myocarditis is infection of the heart muscle by a virus.'' \\
     4: & ``The disease results from an attack by the body’s own immune system, causing inflammation in the walls of arteries.'' \\
     5: & ``The disease disrupts the flow of blood around the body, posing serious cardiovascular complications.'' \\ \midrule
          \vspace{5mm}
    \end{tabular}
    }

    \scalebox{0.8}{
    \begin{tabular}{@{}lp{20cm}@{}} \toprule
     \multicolumn{2}{@{}l@{}}{\textbf{Query}: ``Psyllium fiber comes from the outer coating, or husk of the psyllium plant's seeds.''} \\ \midrule
     \multicolumn{2}{@{}l@{}}{\textbf{Results by nearest neighbor search}} \\
     1: & ``\textcolor{red}{Psyllium is a form of fiber made from the Plantago ovata plant, specifically from the husks of the plant’s seed.''} \\
     2: & ``\textcolor{red}{Psyllium is a form of fiber made from the Plantago ovata plant, specifically from the husks of the plant’s seed.''} \\
     3: & ``Psyllium husk is a common, high-fiber laxative made from the seeds of a shrub.'' \\
     4: & ``\textcolor{blue}{Psyllium seed husks, also known as ispaghula, isabgol, or psyllium, are portions of the seeds of the plant Plantago ovata, (genus Plantago), a native of India and Pakistan.''} \\
     5: & ``\textcolor{blue}{Psyllium seed husks, also known as ispaghula, isabgol, or psyllium, are portions of the seeds of the plant Plantago ovata, (genus Plantago), a native of India and Pakistan.''} \\ \midrule
     \multicolumn{2}{@{}l@{}}{\textbf{Results by diverse nearest neighbor search (proposed)}} \\
     1: & ``Psyllium is a form of fiber made from the Plantago ovata plant, specifically from the husks of the plant’s seed.'' \\
     2: & ``Psyllium husk is a common, high-fiber laxative made from the seeds of a shrub.'' \\
     3: & ``Flaxseed oil comes from the seeds of the flax plant (Linum usitatissimum, L.).'' \\
     4: & ``The active ingredients are the seed husks of the psyllium plant.'' \\
     5: & ``Sisal fibre is derived from the leaves of the plant.'' \\ \midrule
           \vspace{5mm}
    \end{tabular}
    }

    \scalebox{0.8}{
    \begin{tabular}{@{}lp{20cm}@{}} \toprule
     \multicolumn{2}{@{}l@{}}{\textbf{Query}: ``In the United States there are grizzly bears in reserves in Montana, Idaho, Wyoming and Washington.''} \\ \midrule
     \multicolumn{2}{@{}l@{}}{\textbf{Results by nearest neighbor search}} \\
     1: & ``\textcolor{red}{In the United States there are grizzly bears in reserves in Montana, Idaho, Wyoming and Washington.''} \\
     2: & ``\textcolor{red}{In North America, grizzly bears are found in western Canada, Alaska, Wyoming, Montana, Idaho and a potentially a small population in Washington.''} \\
     3: & ``In the United States black bears are common in the east, along the west coast, in the Rocky Mountains and parts of Alaska.'' \\
     4: & \textcolor{blue}{``Major populations of Canadian lynx, Lynx canadensis, are found throughout Canada, in western Montana, and in nearby parts of Idaho and Washington.''} \\
     5: & \textcolor{blue}{``Major populations of Canadian lynx, Lynx canadensis, are found throughout Canada, in western Montana, and in nearby parts of Idaho and Washington.''} \\ \midrule
     \multicolumn{2}{@{}l@{}}{\textbf{Results by diverse nearest neighbor search (proposed)}} \\
     1: & ``In the United States there are grizzly bears in reserves in Montana, Idaho, Wyoming and Washington.'' \\
     2: & ``In the United States black bears are common in the east, along the west coast, in the Rocky Mountains and parts of Alaska.'' \\
     3: & ``Major populations of Canadian lynx, Lynx canadensis, are found throughout Canada, in western Montana, and in nearby parts of Idaho and Washington.'' \\
     4: & ``Today, gray wolves have populations in Alaska, northern Michigan, northern Wisconsin, western Montana, northern Idaho, northeast Oregon and the Yellowstone area of Wyoming.'' \\
     5: & ``There are an estimated 7,000 to 11,200 gray wolves in Alaska, 3,700 in the Great Lakes region and 1,675 in the Northern Rockies.'' \\ \midrule    \end{tabular}
    }
    \caption{Additional qualitative evaluation on text data using MS MARCO.}
    \label{tbl:msmarco2}
\end{table*}

In \cref{tbl:msmarco2}, we present additional results of a diverse search on text data as conducted in Sec 7.5. Here, we introduce the results for three queries as follows.

For the first query, ``This condition...'', three identical results appear at the first three results. When considering RAG, it is typical for the information source to contain redundant data like this. Removing such redundant data beforehand can sometimes be challenging. For instance, if the data sources are continuously updated, it may be impossible to check for redundancy every time new data is added. The proposed LotusFilter helps eliminate such duplicate data as a simple post-processing. LotusFilter does not require modifying the data source or the nearest neighbor search algorithm.

For the second query, ``Psyllium...'', the first and second results, as well as the third and fourth results, are almost identical. This result illustrates that multiple types of redundant results can often emerge during the search. Without using the proposed LotusFilter, removing such redundant results during post-processing is not straightforward.

For the third query, ``In the United...'', while there is no perfect match, similar but not identical sentences are filtered out. We can achieve it because LotusFilter identifies redundancies based on similarity in the feature space. As shown, LotusFilter can effectively eliminate similar results that cannot necessarily be detected through exact string matching.

\end{document}